\PassOptionsToPackage{square,comma,numbers,sort&compress}{natbib} 
\documentclass{article}

\usepackage[final]{neurips_2021}
\usepackage{bbding} 

\usepackage[utf8]{inputenc} 
\usepackage[T1]{fontenc}    
\usepackage{url}            
\usepackage{booktabs}       
\usepackage{amsfonts}       
\usepackage{nicefrac}       
\usepackage{microtype}      
\usepackage{wrapfig}

\usepackage{amssymb}

\usepackage{amsmath}
\usepackage{amsfonts}
\usepackage{algorithm}
\usepackage{algorithmic}

\newtheorem{theorem}{Theorem}
\newtheorem{lemma}{Lemma}
\newtheorem{procedure}{Procedure}
\newtheorem{definition}{Definition}
\newtheorem{note}{Note}

\newtheorem{proof}{Proof}

\usepackage{times}
\usepackage{epsfig}
\usepackage{graphicx}
\usepackage{amsmath}
\usepackage{amssymb}

\usepackage{amsfonts}
\usepackage{algorithm}
\usepackage{algorithmic}
\usepackage{color}

\usepackage{lineno,hyperref}
\usepackage{amsfonts}
\usepackage{multirow}
\usepackage{bm}

\usepackage{diagbox}
\usepackage{multirow}

\usepackage{times}
\usepackage{soul}
\usepackage{url}
\usepackage{algorithm}
\usepackage{algorithmic}
\usepackage{subfigure}
\usepackage{lineno,hyperref}
\usepackage{multirow}
\usepackage{bm}
\modulolinenumbers[1]
\usepackage{rotating}
\usepackage{booktabs}
\usepackage{makecell}

\usepackage{wrapfig,lipsum,booktabs}

\newcommand{\etal}{\textit{et al.} }

\usepackage{amsfonts}       
\usepackage{nicefrac}       
\usepackage{microtype}      
\usepackage{xcolor}         

\title{Human Machine Co-adaption Interface via Cooperation Markov Decision Process System}

\author{
Kairui Guo $^2$, Adrian Cheng$^2$, Yaqi Li$^1$, Jun Li $^2$, Rob Duffield $^2$, Steven W. Su $^1$ $^*$
\AND
\thanks{$^{*}$ The co-responding author.}
\thanks{$^{1}$ College of Artificial Intelligence and Big data for Medical Sciences, Shandong First Medical University
Shandong Academy of Medical Sciences, China.}
\thanks{$^{2}$ University of Technology, Sydney, Australia.}  
}

\begin{document}

\maketitle

\begin{abstract}
This paper aims to develop a new human-machine interface to improve the rehabilitation performance from the perspective of both the user (patient) and the machine (robot) by introducing the co-adaption techniques via model based reinforcement learning.
Previous studies focus more on robot assistance, i.e., to improve the control strategy so as to fulfil the objective of Assist-As-Needed. In this study, we treat the full process of robot-assisted rehabilitation as a co-adaptive process
 or mutual learning process, and emphasize the adaptation of the user to the machine.
To this end, we proposed a Co-adaptive MDPs (CaMDPs) model to quantify the learning rates based on cooperative multi-agent reinforce learning (MARL) in the high abstraction layer of the systems. We proposed several approaches to cooperatively adjust the Policy Improvement among the two agents in the framework of Policy Iteration. Based on the proposed co-adaptive MDPs, simulation study indicates the non-stationary problem can be mitigated by using various proposed Policy Improvement approaches.
\end{abstract}

\section{INTRODUCTION}

For robotic based rehabilitation system, the human-machine interaction (HMI), which can be defined as robotic systems for use by or with humans \cite{goodrich2008human}, is one of the critical enablers for the assimilation between the intelligent machine and the human users \cite{peternel2018robot}. Human-involved systems with advanced HMI platforms have been developed to meet the growing demands in health, education, defence, and industry \cite{gull2020review}. 
To further improve the learning efficiency and receptiveness between machines and users, HMI requires a mutual learning strategy involving both machine- and human-oriented learning \cite{perdikis2020brain}. 
Due to the dynamic interaction between the human and the machine, there is a need for a co-adaptive system that continuously monitors the human adaption and uses this information to design a RL-based mutual learning algorithm to improve the users’ adaptation rate.

Recent projects have used co-adaptive assistance strategy with biomechanical data during walking in an ankle exoskeleton \cite{collins2019extremes}. An adaptive switching model for upper limb movement is also developed \cite{huang2021human}. Despite such success, in stroke rehabilitation process, there is a need for co-adaptive control strategies with complex motions, especially for specific and purposeful motor skills. Moreover, to further add the high-dimensional features from the user’s adaptation process into the HMI system, reinforcement learning for multiple agents is an optimal selection to develop the human-machine collaboration policy \cite{kim2020adaptive}.



The reinforcement learning-based Multi-Agent cooperative system has been extensively explored in recent years \cite{wong2021multiagent}\cite{nguyen2020deep}. For example, the framework of Cooperation Markov Decision Process (CMDP), which is suitable for the learning evolution of cooperative decision between two agents, has been investigated \cite{mo2020convergence}. This motivated the establishment of a model to specifically supervise the stroke rehabilitation process to guide the patient’s action/behavior patterns and improve the control policy of the robot to accelerate the user adaptation.



In this study, we will propose a specific Co-adaptive MDPs (CaMDPs) to build a framework to handle the co-adaptation between the user and robot. The proposed framework/strategy can be easily extended to systems with more complex environment models by using the model based reinforcement learning techniques \cite{plaat2021high} via state abstraction and temporal abstraction \cite{moerland2020model} \cite{qingji2008robot} \cite{wang2007emotion}. 


For MARL, because the agents simultaneously improve their policies based on their own benefits, the environment from the point view of individual agent turns to be non-stationary and unexplainable during mutual learning.

Other challenges for multi-agents RL also include computational complexity, partial observability and credit assignment \cite{yang2020overview} \cite{nguyen2020deep} \cite{wong2021multiagent}.

A special consideration for the application of MARL in stroke rehabilitation is related to the policy adaptation frequency. Like other medical applications where actions correspond to treatments, it is often unrealistic to execute fully adaptive RL algorithms; instead one can only run a fixed policy approved by the domain experts to collect data, and a separate approval process is required every time \cite{bai2019provably}
\cite{almirall2012designing} \cite{almirall2014introduction} \cite{lei2012smart}. Also, in personalized/individulized  recommendation \cite{theocharous2015personalized}, it is computationally impractical to adjust the policy online based on instantaneous data. A more common practice is to aggregate data in a long period before deploying a new policy.  In many of these applications, adaptivity turns out to be really the bottleneck\cite{bai2019provably}.
Recently, limited adaptability has been viewed as a constraint for designing RL algorithms, which is conceptually similar to those in constrained MDP \cite{dann2019policy} \cite{yu2019convergent}, but it mainly considers the case the frequency of Policy Improvement with an upper bound.

In this paper, we consider the implementation of co-adaptation between human and machine with a focus on robot assisted rehabilitation. The contributions of the paper are summarized as follows:\\
1. We proposed a Co-adaptive MDPs (CaMDPs) model for robot assisted rehabilitation system.\\
2. We analysed the asymptotic characteristics of the value functions of the proposed CaMDPs.\\
3. We proposed a revised Policy Improvement procedure to reduce the frequency of the policy adaptation rate with bounded value-loss for Agent$_0$ (patient).\\
4. We proposed approaches to balance the policy adaptation rate of both two agents for the proposed CaMDPs model.

\section{The Proposed Co-operative MDPs Model for Co-adaptation}


For robot assisted rehabilitation scenario, we define two agents here. Agent$_0$ represents the patient and Agent$_1$ represents the robot. We will reduce the switching cost of Agent$_0$, i.e., the less switching of policy the better.

Due to the limited communication between humans and robots, the overall structure should be a decentralized partially observable system in real-time rehabilitation training. However, at the early stage of co-adaptive development, the decentralized output feedback configuration is not our primary concern. Here, we assume the observer design is applied and the states are all directly observable by either of the two agents.

For different classes of decentralized controlled cooperative problems, several MDPs models have been presented for various special considerations. Goldman \cite{goldman2004decentralized} investigated different communication manners among agents by introducing costs. As an initial research step for the special co-adaptive strategy \cite{mo2020convergence}, we start from simple MDPs as follows.

Assume the model of the two agents MDPs is $M=<S, \, A_0,\, A_1, P, R>$, where
\begin{enumerate}
\item $S$ is a finite set of states.
\item $A_0$ and $A_1$ are finite sets of control actions of the two agents to be considered.
\item $P$ is the transition probability function. $P(s'|s, a_0, a_1)$ is the probability of moving from state $s \in S$ to state $s' \in S$ when agents 0 and 1 perform actions $a_0$ and $a_1$
respectively. We note that the transition model is stationary, i.e., it is independent of
time.
\item $R$ is the global reward function. $R(s, a_0, a_1, s')$ represents the reward obtained by the
system as a whole, when agent 0 executes action $a_0$ and agent 1 executes action $a_1$ in state $s$ resulting in a transition to state $s'$.
\end{enumerate}
Here, similar as the arrangements of  \cite{goldman2004decentralized}, we consider three sets of states: the first set of states are the states $S_{0_i}$ $i \in \{1,2, \cdots, Ns_0\}$ controlled by the patient only; the second set of states are the states $S_{s_i}$ $i \in \{1,2, \cdots, Ns_s\}$ influenced by both the robot and the patient; the third set of states are the states $S_{1_i}$ $i \in \{1,2, \cdots, Ns_1\}$ controlled by the robot only.
It should be emphasized that when the dimensions of the state $S_{0_i}$ and $S_{1_i}$  are both small, the co-adaption of the two-agent becomes more transparency \cite{9462035} \cite{Boyd_2022} and easier as the problem becomes closer to a single-agent learning problem. To reduce the dimensionality of the two sets of states we can use more sensors to generate the communication channels, e.g., the patient emotional features could be extracted from physiological signals (e.g., EEG, ECG) and/or computer vision based signals, and shared by both the two agents.

We give a more formal definition of the system as follows:

\begin{definition} \label{df_1}
A two agents MDPs is a Co-Adaptive MDPs (CAMDPs) system if the set $S$ of states can be factored into three components $S_0$, $S_1$, and $S_s$ such that:\\
$\forall$ $s_0,s'_0 \in S_0$, $\forall$ $s_s,s'_s \in S_s$, $\forall$ $s_1,s'_1 \in S_1$, we have\\
$P(s'_0|(s_0,s_s,s_1),a_0,a_1)=P(s'_0|s_0,a_0)$;\\
$P(s'_s|(s_0,s_s,s_1),a_0,a_1)=P(s'_s|s_s,a_0,a_1)$;\\
$P(s'_1|(s_0,s_s,s_1),a_0,a_1)=P(s'_1|s_1,a_1)$;\\
$R(s_0,a_0,a_1,(s'_0,s'_s,s'_1))=R(s_0,a_0,s'_0)$;\\
$R(s_s,a_0,a_1,(s'_0,s'_s,s'_1))=R(s_s,a_0,a_1,s'_s)$;\\
$R(s_1,a_0,a_1,(s'_0,s'_s,s'_1))=R(s_1,a_1,s'_1)$.

In other words, both the transition probability $  P$ and the reword function $  R$ of the CAMDPs can be represented as\\
$  P = P_0 \bigotimes P_s \bigotimes P_1$, where $P_0 = P({s'}_0|s_0,a_0)$, $P_s = P({s'}_s|s_s, a_s)$, and $P_1 = P({s'}_1|s_1, a_1)$ and\\
$  R = R_0 \bigotimes R_s \bigotimes R_1$, where $R_0 = R(s_0,a_0,{s'}_0)$, $R_s = R(s_s,a_0,a_1,{s'}_s)$, and $R_1 = R(s_1,a_1,{s'}_1)$, where "$\bigotimes$" stands for Kronecker product.
\end{definition}

Here, we did not consider the states which cannot be controlled by either Agent$_0$ or Agent$_1$, but these states might influence the states $S_0$, $S_s$, and/or $S_1$. For example, the weather could not be controlled by either the patient or the robot, but it might influence the emotion of the patients. In the future, if it is necessary, the uncontrollable states could also be included in the CAMDPs.

As discussed before, based on $P$ and $R$ for each individual subsystems, we can construct the augmented $ P$ and $  R$ for the overall system via Kronecker product, i.e., $  P= P_0 \bigotimes P_s \bigotimes P_1$. However, in some cases, some of the three state sets could be empty, but the augmented $P$ and $R$ can still be constructed by the rest sets of states via Kronecker product.

To build the CAMDPs model for different rehabilitation situations, we need to select the proper configurations (decentralized vs centralized) first. Then determine the major components of the system, including the states, the actions, and the reward functions.

There are two different types of rehabilitation scenarios: hospital-based rehabilitation and home-based rehabilitation \cite{lopez2015home}. For hospital-based rehabilitation, since the availability of various medical equipment and the doctor's supervision, state information is often available for both agents. Then, the centralized analysis/training configuration can be implemented. On the other hand, for home-based rehabilitation, due to lacking sufficient monitoring of medical equipment, the supervision of doctors, and fast communication channels, the decentralized configuration should be considered. Here, we consider the decentralized CaMDP model, where the reward function, the state information, and actions, can only be available for the patient and robot separately (see the discussion of the decentralized record function settings presented in \cite{goldman2004decentralized}).

The construction of CaMDPs for rehabilitation is an intricate procedure, which requires the collaboration of medical professionals and system engineers to select model parameters (e.g. the states, control actions, and reward functions).   The low-layer model of rehabilitation exercise would be very complex as both robot and the patient are part of the system. Based on the parameter sensitivity analysis technique discussed in \cite{gold2021addressing} \cite{srikrishnan2022uncertainty},  the low layer of the model can be simplified under the consultation of medical professionals.

The proposed CaMDPs is not designed for the low-layer of the robot-assisted rehabilitation system. Actually,  by using abstraction techniques, a hierarchical reinforcement learning \cite{hutsebaut2022hierarchical} configuration can be constructed with the CaMDPs as the high layer of the rehabilitation system.

To further simplify the discussion, in this study, we only consider $finite$ MDPs \cite{ariasmdps}. One reason is that in rehabilitation engineering,  the rating of patients is often simply by using an approximated integer number.  For example, NIHSS (The National Institutes of Health Stroke Scale) as the gold standard for clinical stroke assessment and measurement rates the degree of severity for stroke patients by using an integer number (e.g., 0 to 4 for the assessment of the motion of the arm). Although some physiological signals (e.g., EMG and EEG) have been applied for personalized stroke assessment,  the severity of the patients can still be assessed by using finite discrete values.

In addition, the action space of the proposed CaMDPs is also in finite dimension. One possible approach for the discretization of the action space is to construct a hierarchical control structure. In a two-layer configuration, the high layer of the control governer the switching of the pre-designed low layer controllers as discussed in \cite{huang2021human}. The human experts' knowledge could be integrated into the system \cite{nguyen2018human} \cite{maadi2021review} via the design of the subsystems, e.g., the sub-controllers.

In summary, the CaMDPs is designed to handle the high layer of robot-assisted rehabilitation with finite states and control actions.

\section{Preliminary}


\label{Sec_con_view}


Based on the augmented $ P$ and $ R$, we present the following lemmas for exploring the asymptotic characteristics of value function to facilitate our analysis.

\begin{lemma}
Assume a transition matrix $P$ is an irreducible aperiodic stochastic matrix (i.e., quasi-positive or ergodic \cite{garcia2013markov}).
Then, all column vectors of the following matrix
$$
\lim_{n \to \infty} \lim_{\gamma \to 1^{-}} (\frac{1}{n} \sum_{i=0}^{n}(\gamma^i P^i))
$$
will approach the same values.
\end{lemma}

\begin{lemma} \label{Lm1}
For three possibility transition matrices, $A$, $B$, and $C$, the augmented possibility transition matrix:
$$A \bigotimes B \bigotimes C $$
is an irreducible aperiodic  stochastic matrix iff $A$, $B$, and $C$ are irreducible aperiodic  stochastic matrices.
\end{lemma}

For rehabilitation exercise, as it is hard to specifically select a particular initial state, we propose the following lemma to analyze the value functions regarding the initial condition of the state.

\begin{lemma} \label{Lm2}
For the CaMDPs model, assume all the probability transition matrix of the three subsystems are quasi-positive. Then, for any two control policy $\pi_0$ and $\pi_1$ ($\pi=\{\pi_0, \pi_1 \}$), the value function $V(s_i)$ will converge to:
$$
[\bold{I} - \gamma {P}]^{-1} (diag( P  R'))= \lim_{n \to \infty}  (\sum_{i=0}^{n}(\gamma^i  P^i)) (diag( P  R')),
$$
where $0 < \gamma < 1$ is the discount factor. Also, when $\gamma$ approaches to $1$, $V(s_i)$ (for $\forall s_i$) will converge to the same value:
$$
g^{\pi}=\lim_{n \to \infty} \lim_{\gamma \to 1^{-}} (\frac{1}{n} \sum_{i=0}^{n}(\gamma^i  P^i)) (diag( P  R')).
$$
\end{lemma}

\section{The adjustment of policy improvement for CAMDPs}
\label{Sec_Switching}

As discussed, nonstationary is one of the most challenges for multi-agent reinforcement learning (MARL).
For zero-sum MARL, Mazumdar et al. \cite{mazumdar2019policy} showed the convergence result of single-agent policy gradient methods is provably non-convergent in simple linear-quadratic games.  The reason is that the agents concurrently improve their policies according to their own interests. The individual agent cannot tell whether the state transition or the change in reward is an actual outcome due to its own action or if it is due to other agents' explorations.

For cooperative reinforcement learning of CAMDPs investigated in this study, the two agents can be coordinated to adjust their policies if $cheap$ communication channels \cite{goldman2004decentralized} are available. In addition, to well handle the non-stationary problem, we can design a two-layer hierarchical learning framework, and implement various switching strategies at the high layer. Especially for the proposed CAMDPs model, we believe that an "intelligent" switching algorithm can be developed in the high layer to significantly reduce the bad effect of nonstationary. For example, we may design a special switching law so that the two agents, instead of performing their Policy Improvement procedure simultaneously,  alternatively improve their policy. In this case, the mutual influences due to policy adjustment are decreased, and the nonstationary phenomenon could be remedied.

However, in some cases, constructing an optimal switching law in the high layer could be time-consuming or unnecessary. In this study, instead of pre-design a switching law,  we introduce new methods to enable the two agents "automatically" adjust the updation frequency of their Policy Improvement procedures and ensure the convergence of the mutual learning process.

In the following two subsections, we introduce the proposed self-adjustment switching methods for CAMDPs, which are extendable to other multi-agent reinforcement learning problems.

%
%

\subsection{Revised Policy Improvement procedure for Agent$_0$}

As discussed in Introduction section, for medical applications where policies
correspond to treatments \cite{dann2019policy} \cite{yu2019convergent}, it is not feasible to $frequently$ switch the policy of Agent$_0$ (i.e. the patient). To reduce this frequency, we first introduce a new Policy Improvement procedure for Agent$_0$. We will show that this procedure can significantly decrease the switching frequency with a predefined bounded value loss.

To give a more intuitive introduction to the proposed switching frequency tuning approach, motivated by the tuning of the second-order system, we borrow the phrase "damping ratio" to represent the degree of tunability of the switching frequency. To tune the "damping ratio" of the system, i.e., adjusting the switching frequency, we will introduce new approaches.  The key to these approaches is a revised Policy Improvement procedure (Procedure \ref{prd_1} below).  Compared with classical reinforcement learning, via a predefined threshold value ($\eta$), the revised Policy Improvement procedure can reduce the switching frequency with the bounded value-loss proportional to $\eta$.


\begin{procedure} \label{prd_1} (Revised Policy Improvement procedure)
Policy Improvement\\
$policy-stable \longleftarrow true$\\
For each $s \in S$, $k$ and a given $\eta$:

$\,\,\,\,\,\, temp \longleftarrow \pi_{k}(s)$

$\,\,\,\,\,\,\,\,\,\,\,\,$ Under policy $\pi_{k}$ calculate

$\,\,\,\,\,\,\,\,\,\,\,\,J_k= \sum_{s'_0,r} p(s',r|s,a)[r^{a}_{s,s'}+\gamma V(s')]$.

$\,\,\,\,\,\,\,\,\,\,\,\,$ If $\underset{a}{\max} (\sum_{s'_0,r} p(s',r|s,a)[r^{a}_{s,s'}+\gamma V(s')]) - J_k \ge \eta$,

$\,\,\,\,\,\,\,\,\,\,\,\,\,\, \pi(s) \longleftarrow \underset{a}{\arg\max} (\sum_{s'_0,r} p(s',r|s,a)[r^{a}_{s,s'}+\gamma V(s')] ). $

$\,\,\,\,\,\,$ If $ temp \ne \pi(s)$, then $policy-stable \longleftarrow false$\\
If $policy-stable$, then stop and return $V$ and $\pi$; else go to Policy Evaluation step.
\end{procedure}

\begin{theorem} \label{m_theorem}
For a CaMDPs, we assume it is ergodic under all policies ${\pi \in \Pi}$. If the Agent$_0$ is adjusted according to Revised-Policy-Improvement (Procedure \ref{prd_1}), and the Agent$_1$ is performing under the policy $\pi_1^j$, then, the value loss $\bold{\delta V}$ will be less than $\eta [\bold{I} - \gamma \bold{P}^{\pi^*}]^{-1} \bold {1}$, where $\bold{P}^{\pi^*}$ is the state probability transition matrix under optimal policy $\pi^*$, and $\bold {1}$ is the all one vector.
\end{theorem}

\begin{proof}
If we treat the overall CAMDPs as a single agent MDPs, and its probability transition matrices and reward matrices under policy $\pi_0$ and $\pi_1$ can be constructed based on those of the sets of states $S_0$, $S_s$ and $S_1$. Then, we consider the case that $\pi_1$ is fixed on its $j$-th policy, i.e., $\pi_1^j$. Then, it is time for the policy improvement of Agent$_0$. If assume the current policy for the Agent$_0$ is $\pi_0^m$, and under the classical policy improvement procedure the selected policy is $\pi_0^n$. We denote the augmented two policies for the overall system as follows:
\begin{equation*}
\begin{array}{cc}
     \pi^m=\{\pi_0^m, \pi_1^j\} & \pi^n=\{\pi_0^n, \pi_1^j\}
\end{array}
\end{equation*}

Following the classical policy-improvement routine \cite{howard1960dynamic}, since $\pi_0^n$ was chosen over $\pi_0^m$, we have
$$
r_i^{\pi^n} + \gamma \sum_{j=1}^N p_{ij}^{\pi^n} v_j^{\pi^m} \ge
r_i^{\pi^m} + \gamma \sum_{j=1}^N p_{ij}^{\pi^m} v_j^{\pi^m}.
$$
where $i \in \{1, 2, \cdots, N\}$ and $N$ is the total number of states of the augmented system.

For the two combined policies $\pi^m$ and $\pi^n$ and each state $s_i$, we have the following equations:

\begin{equation} \label{eq0a}
\begin{array}{cc}
v_i^{\pi^m} &= r_i^{\pi^m} + \gamma \sum_{j=1}^N p_{ij}^{\pi^m} v_j^{\pi^m}, \\
 v_i^{\pi^n}& = r_i^{\pi^n} +\gamma \sum_{j=1}^N p_{ij}^{\pi^n} v_j^{\pi^n}.
\end{array}
\end{equation}

Considering policy improvement procedure, we define for each particular state $s_i$:
\begin{equation} \label{eq1a}
g_i=r_i^{\pi^n} + \gamma \sum_{j=1}^N p_{ij}^{\pi^n} v_j^{\pi^m} -
r_i^{\pi^m} - \gamma \sum_{j=1}^N p_{ij}^{\pi^m} v_j^{\pi^m}
\end{equation}
It is clear that $\forall i, g_i \ge 0$.
Furthermore, under the policies $\pi^m$ and $\pi^n$, based on both equations (\ref{eq0a}) and (\ref{eq1a}), we have
\begin{equation}
    v_i^{\pi^n}-v_i^{\pi^m}=g_i + \gamma \sum_{j=1}^{N} p_{ij}^{\pi^n} (v_j^{\pi^n}-v_j^{\pi^m}).
\end{equation}
Defining $\delta v_i = v_i^{\pi^n}-v_i^{\pi^m}$, we have
\begin{equation}
    \delta v_i = g_i + \gamma \sum_{j=1}^{N} p_{ij}^{\pi^n} \delta v_j
\end{equation}
As it assumed the overall CAMDPs is ergodic under all policies, we know the vector solution of the above equation is as follows:
\begin{equation}
    \bold{\delta V} = [\bold{I} - \gamma \bold{P}^{\pi^n}]^{-1} \bold{g}.
\end{equation}
It can be seen that  $\bold{\delta V} \ge 0$ as the elements of $[\bold{I} - \gamma \bold{P}^{\pi^n}]^{-1}=\sum_{k=0}^{+\infty}\gamma^k {(\bold{P}^{\pi^n})}^k$ are all {\bf non-negative}. Furthermore, if $g_i \le \eta$, then we have
$$\bold {\delta V} \le \eta [\bold{I} - \gamma \bold{P}^{\pi^n}]^{-1} \bold {1}.$$
If the policy $\pi^n$ is optimal, we can see the value loss will be no great than $\eta [\bold{I} - \gamma \bold{P}^{\pi^*}]^{-1} \bold {1}=\frac{\eta}{1-\gamma}$.
\end{proof}

\begin{note}
Here, we only consider the case that the Agent$_1$ is performing the control policy $\pi_1^j$ ($j \in \{1, \cdots, N_{action_1} \}$ ($N_{action_1}$ is the number of action options of Agent$_1$). To ensure the overall loss is less than $\epsilon$, we need to consider the cases $\forall \pi_1^j$, $j \in \{1, \cdots, N_{action_1} \}$.
\end{note}

\subsection{Self-adjusted switching adaptation strategies for CaMDPs}
\label{sec:Exp}

To address the nonstationary issue for the proposed CaMDPs, we introduce a practical method based on Revised Policy Improvement procedure to improve the convergence of the CaMDPs. We illustrate this idea via the step response of a classical second-order system.  For a second-order system, if its damping ratio "$\zeta $" decreases, it will respond fast but with less stability margin. On the other hand, if its damping ratio increases, the system will respond slowly but with a big stability margin. For the two agents of the CaMDPs, if one of the agents responds much fast than the other (normally, we anticipate Agent$_1$ (robot) has a faster response than Agent$_0$ (patient) in terms of policy optimization), then the chance of non-stationary will be lower. Based on this idea, we provide the following two methods to adjust the policy improvement rate (i.e., $damping$ $ratio$ ) of the two agents so that they can work harmonically.

Based on the revised Policy Improvement procedure (Procedure \ref{prd_1}), a threshold $\eta$ is predefined based on the tolerable value loss. This fixed threshold $\eta$ is similar to a $proportional$ measurement of value loss. As in most cases, the full model, as well as the asymptotic value, could not be obtained at the beginning of the learning procedure, to pre-select a proper $\eta$ is unrealistic. An alternative option is to select a relatively big $\eta$ and define an $integral-alike$ parameter. Together with the pre-selected $proportional-alike$ parameter, we can then construct a proportional and integral type of policy improvement scheme to tune the policy improvement frequency so that the $damping$ $ratio$ of the system can be tuned with the desired pace. We can fine-tune the two parameters during the learning process. We call this approach a "PI-alike" adjustment scheme.

Specifically, assume the integral-alike parameter is "$\kappa_I$", we can modify the switching condition in Procedure \ref{prd_1} as follows:

\begin{procedure} \label{prd_2} (PI-alike Policy Improvement procedure)
Policy Improvement\\
$policy-stable \longleftarrow true$\\
For each $s \in S$, $k$, and the pre-selected parameters $M$, $\eta$ and $\kappa_I$:

$\,\,\,\,\,\, temp \longleftarrow \pi_{k}(s)$

$\,\,\,\,\,\,\,\,\,\,\,\,$ Under policy $\pi_{k}$ calculate

$\,\,\,\,\,\,\,\,\,\,\,\,J_k= \sum_{s'_0,r} p(s',r|s,a)[r+\gamma V(s')]$.

$\,\,\,\,\,\,\,\,\,\,\,\,$
Calculate $I_k=\underset{a}{\max} (\sum_{s'_0,r} p(s',r|s,a)[r+\gamma V(s')]) - J_k$.

$\,\,\,\,\,\,\,\,\,\,\,\,$
if $\kappa_I \sum_{j=0}^{M} I_{k-j} \ge \eta$,

$\,\,\,\,\,\,\,\,\,\,\,\,\,\, \pi(s) \longleftarrow \underset{a}{\arg\max} (\sum_{s'_0,r} p(s',r|s,a)[r+\gamma V(s')] ). $

$\,\,\,\,\,\,$ If $ temp \ne \pi(s)$, then $policy-stable \longleftarrow false$\\
If $policy-stable$, then stop and return $V$ and $\pi$; else go to Policy Evaluation step.
\end{procedure}


The implementation of PI-alike tuning method needs to modify the classical Policy Improvement procedure. We will show this approach by using numerical simulation in Section \ref{nu_e}.

\section{Numeral Analysis} \label{nu_e}
In this section, we consider a simple example for the co-adaption between rehabilitation robot and the human user.

Let us consider a CaMDPs system $M=<S, \, A_0,\, A_1, P_0, P_1,R_0, R_1>$, with details as follows:

The state is the Kronecker product of three state sets:
$S_0=\{s_{00}, s_{01}\}$,
$S_s=\{s_{s0}, s_{s1}\}$, and
$S_1=\{s_{10}, s_{11}\}$.

The action set is the Kronecker product of two sets:
$A_0=\{a_{0_0}, a_{0_1}\}$ and $A_1=\{a_{1_0}, a_{1_1}\}$.

The probability transition matrices under different control actions for each subsystem are:
\begin{equation*}
    P_0(S_0,a_{0_0},S_0)=\left[\begin{array}{cc}
   0.8229  &  0.1771\\
    0.7826 &   0.2174
       \end{array}\right];
\end{equation*}
\begin{equation*}
    P_0(S_0,a_{0_1},S_0)=\left[\begin{array}{cc}
    0.6406  &  0.3594\\
    0.4919  &  0.5081
    \end{array}\right];
\end{equation*}

\begin{equation*}
    P_s(S_s,a_{0_0},a_{1_0},S_s)=\left[\begin{array}{cc}
     0.5821  &  0.4179\\
    0.3839   & 0.6161
    \end{array}\right];
\end{equation*}
\begin{equation*}
    P_s(S_s,a_{0_0},a_{1_1},S_s)=\left[\begin{array}{cc}
 0.1838  &  0.8162\\
    0.5686 &   0.4314
    \end{array}\right];
\end{equation*}
\begin{equation*}
    P_s(S_s,a_{0_1},a_{1_0},S_s)=\left[\begin{array}{cc}
    0.6990  &  0.3010\\
    0.6169  &  0.3831
    \end{array}\right];
\end{equation*}
\begin{equation*}
    P_s(S_s,a_{0_1},a_{1_1},S_s)=\left[\begin{array}{cc}
      0.3448 &   0.6552\\
    0.6432   & 0.3568
    \end{array}\right];
\end{equation*}

\begin{equation*}
    P_1(S_1,a_{1_0},S_1)=\left[\begin{array}{cc}
   0.8022  &  0.1978\\
    0.5396 &   0.4604
    \end{array}\right];
\end{equation*}
\begin{equation*}
    P_1(S_1,a_{1_1},S_1)=\left[\begin{array}{cc}
    0.4083  &  0.5917\\
    0.5815  &  0.4185
    \end{array}\right].
\end{equation*}

The reward function are as follows:
\begin{equation*}
    R_0(S_0,a_{0_0},S_0)=\left[\begin{array}{cc}
   0.1565  &  0.1769\\
    0.1909 &   0.1425
    \end{array}\right];
\end{equation*}
\begin{equation*}
    R_0(S_0,a_{0_1},S_0)=\left[\begin{array}{cc}
     0.0520  &  0.2813\\
    0.1530   & 0.1803
    \end{array}\right];
\end{equation*}

\begin{equation*}
    R_s(S_s,a_{0_0},a_{1_0},S_s)=\left[\begin{array}{cc}
   0.2136  &  0.1197\\
    0.1533 &   0.1800
    \end{array}\right];
\end{equation*}
\begin{equation*}
    R_s(S_s,a_{0_0},a_{1_1},S_s)=\left[\begin{array}{cc}
 0.3047  &  0.0286\\
    0.0895  &  0.2438
    \end{array}\right];
\end{equation*}
\begin{equation*}
    R_s(S_s,a_{0_1},a_{1_0},S_s)=\left[\begin{array}{cc}
    0.0077 &   0.3257\\
    0.1378 &   0.1955
    \end{array}\right];
\end{equation*}
\begin{equation*}
    R_s(S_s,a_{0_1},a_{1_1},S_s)=\left[\begin{array}{cc}
     0.2806 &   0.0527\\
    0.1625  &  0.1708
    \end{array}\right];
\end{equation*}

\begin{equation*}
    R_1(S_1,a_{1_0},S_1)=\left[\begin{array}{cc}
    0.0190  &  0.3144\\
    0.3120  &  0.0213
\end{array}\right];
\end{equation*}
\begin{equation*}
    R_1(S_1,a_{1_1},S_1)=\left[\begin{array}{cc}
    0.1878  &  0.1455\\
    0.0450  &  0.2883
    \end{array}\right].
\end{equation*}

\subsection{The asymptotic characteristics for values under different discount factors }
When the discount factor $\gamma=0.5$, see the bottom part of Fig.\ref{fig_1}. The asymptotic values of each states are relatively different with each other. However, when the discount factor close to $1$, see the top part of Fig.\ref{fig_1} ($\gamma=0.98$). The asymptotic values of each states are quite close. This is consistent with Lemma \ref{Lm1}. Based on this property, we will provide a better policy without considering the inter-states differences.

\begin{figure}[ht]
\centering{\includegraphics[width=0.45\textwidth]{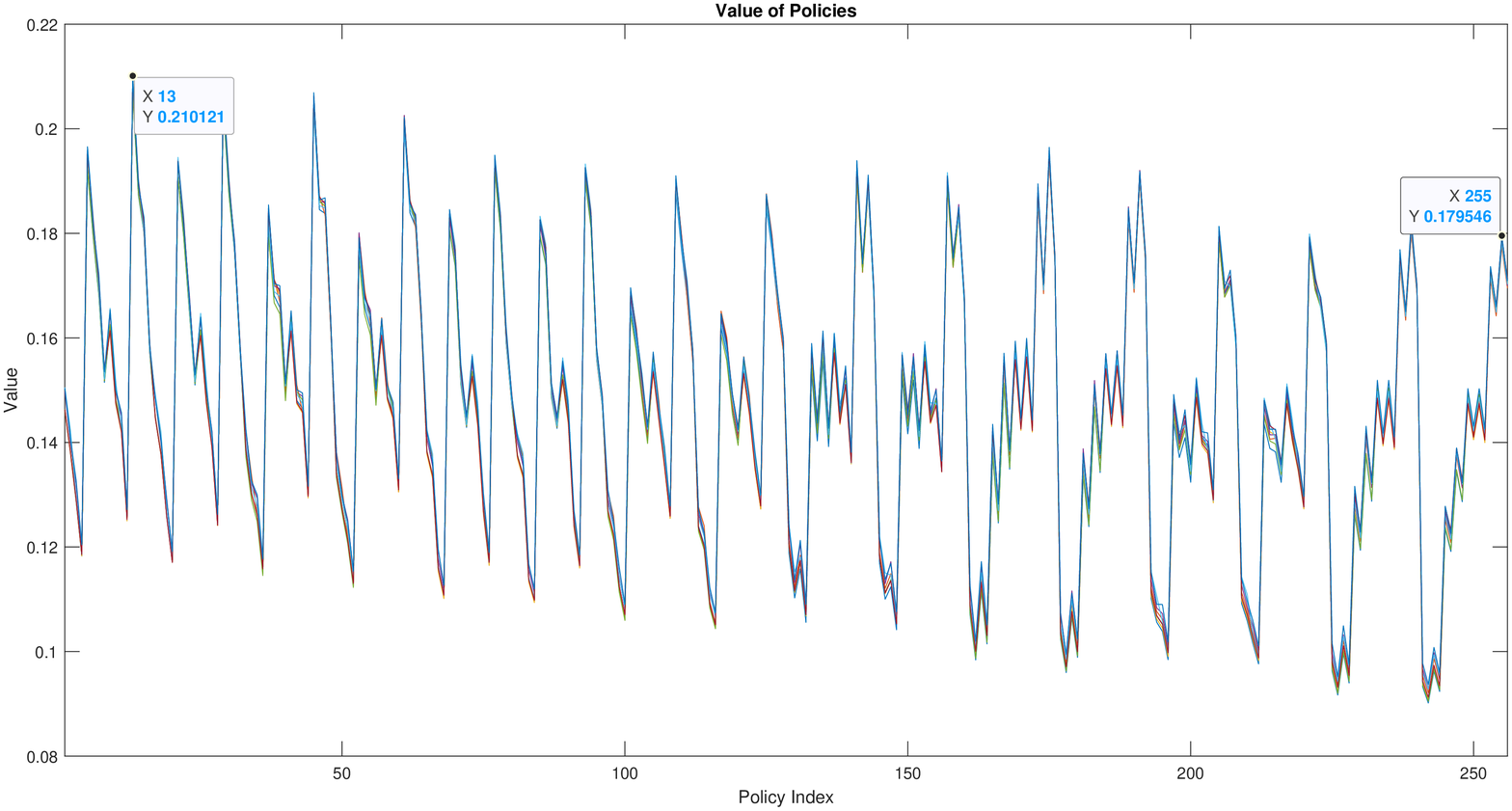}}
\centering{\includegraphics[width=0.45\textwidth]{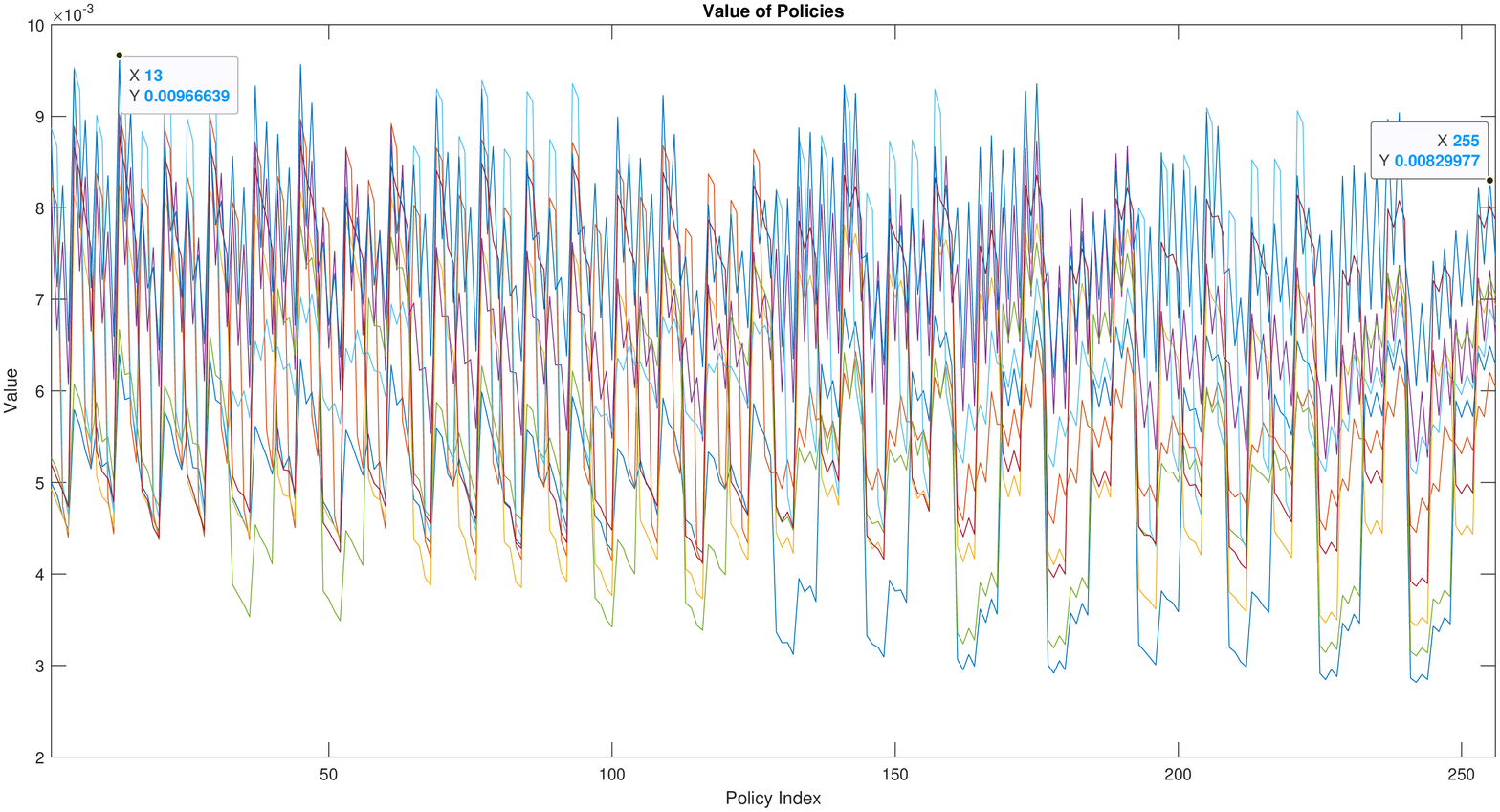}}\\
\caption{Value vs state under different discount factors:\\ a. $\gamma=0.98$ (left). b. $\gamma=0.50$ (right). }\label{fig_1}
\end{figure}

\subsection{The value loss under the revised policy improvement procedure}

According to Procedure \ref{prd_1} (Revised   Policy   Improvement   Procedure), and Theorem \ref{m_theorem}, we perform the numerical analysis.

The optimal policy ($value=0.2101$), based on the asymptotic analysis of Lemma, can be calculated as follows:

For Agent$_0$, it is
\begin{equation*}
\left[
\begin{array}{ccccc}
State: &\{s_{0_0}, s_{s_0}\}& \{s_{0_0}, s_{s_1}\}& \{s_{0_1}, s_{s_0}\}& \{s_{0_1}, s_{s_1}\} \\
Action_0:& 0& 0& 0& 0
\end{array}
\right].
\end{equation*}

For Agent$_1$, it is
\begin{equation*}
\left[
\begin{array}{ccccc}
State: &\{s_{1_0}, s_{s_0}\}& \{s_{1_0}, s_{s_1}\}& \{s_{1_1}, s_{s_0}\}& \{s_{1_1}, s_{s_1}\} \\ Action_1:& 1& 1& 0& 0
\end{array}
\right].
\end{equation*}
For simplicity, when there is no confusion, we use $\pi_0=[0\,0\,0\,0]$ and $\pi_1=[1\, 1\, 0\, 0]$, or $[0\,0\,0\,0]$ $[1\, 1\, 0\, 0]$ for short.

If we set the initial control policy as $\pi_0=[1\,1\,1\,1]$ and $\pi_1=[1\,1\,0\,0]$ ($value=0.1736$), and Agent$_0$ adopts the Procedure \ref{prd_1} while Agent$_1$ still adopts the normal Policy Improvement strategy.

When we select different $\eta$ values, the Agent$_0$ can either stay in the same policy (to avoid any switching of control policy), or partly switching the policy with respect to some states, or fully switching to the optimal policy. To see the detailed results, see Table \ref{tb_1}.

From Table \ref{tb_1}, it can be seen that the $\eta$ value should be well selected in order to avoid nonstationary. However, in the beginning, as the lacking of the knowledge of the process, pre-select a suitable $\eta$ will be challenging. Here, we applied the PI-alike approach as described by Procedure \ref{prd_2}. We simply select $\kappa_I=1$ and $\eta=0.1$; the nonstationary of policy improvement has been addressed, and they converged to the optimal policy $\pi_0^*=[0\,0\,0\,0]$ and $\pi_1^*=[1\,1\,0\,0]$ ($value=0.210$) as shown in fig. \ref{fig_2}.

\begin{figure}[bt]
\centering{\includegraphics[width=0.45 \textwidth]{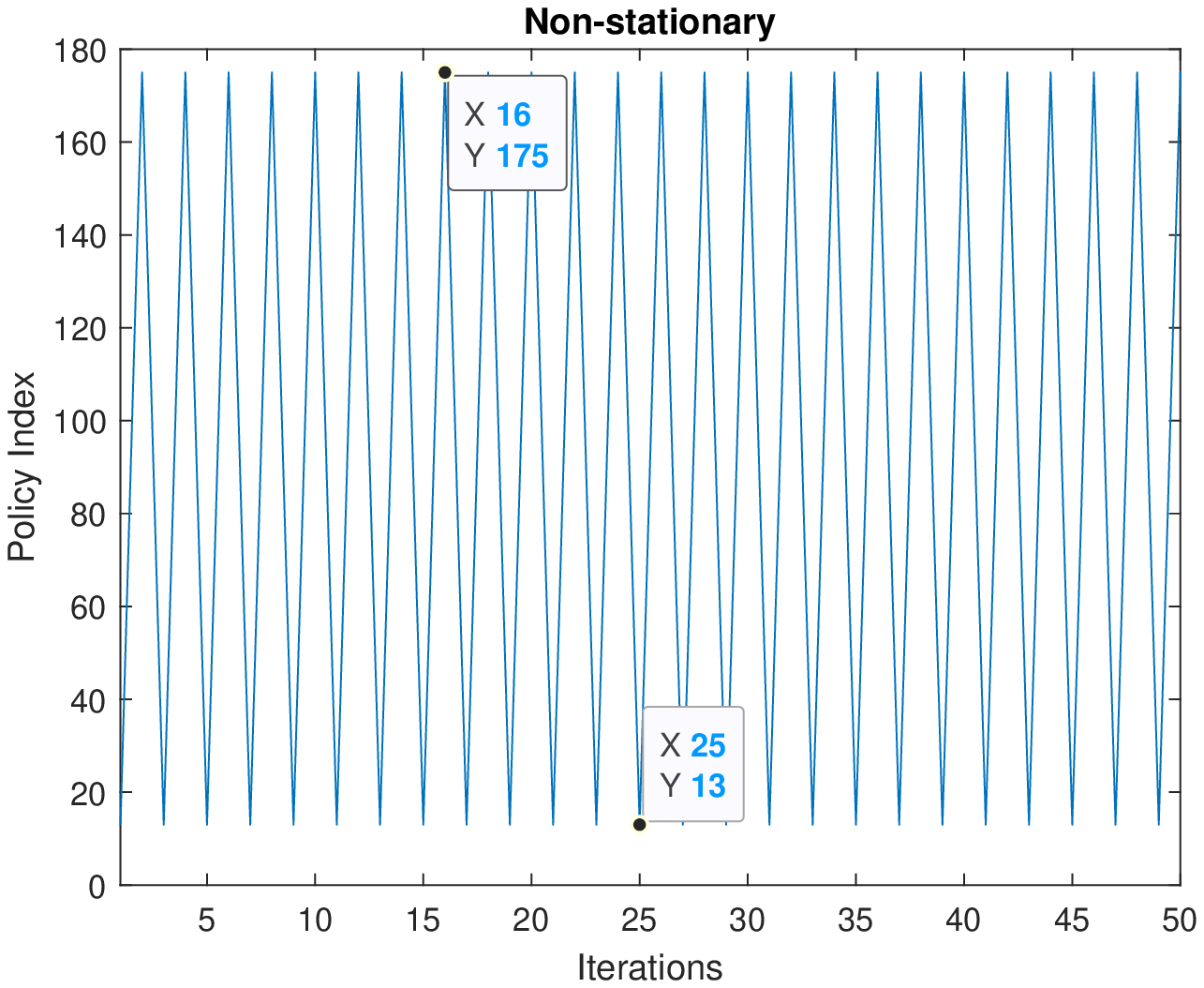}}
\centering{\includegraphics[width=0.45 \textwidth]{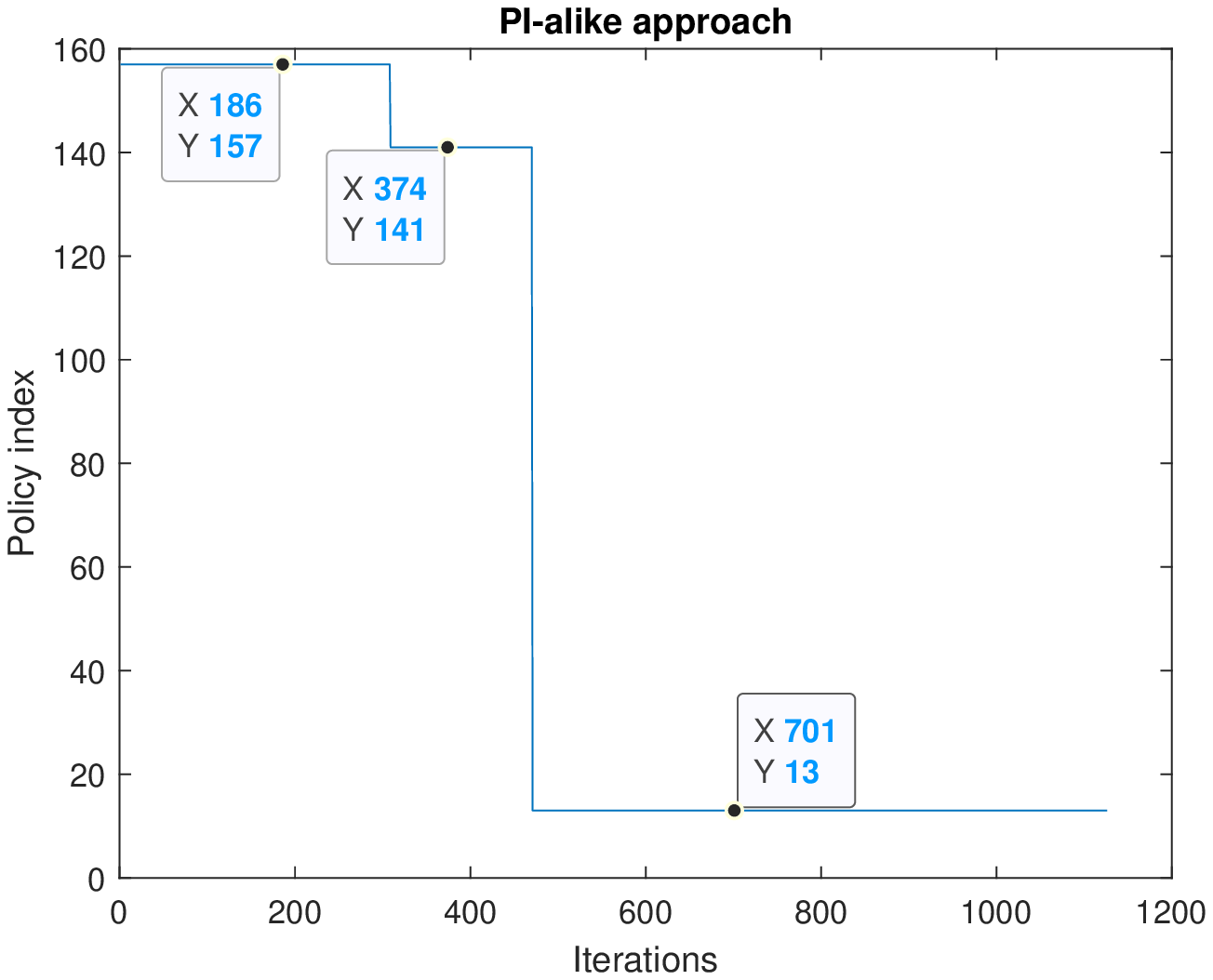}}\\
\caption{The convergence of Policy Improvement: a. Non-stationary. b. PI-alike approach for the improvement of the convergence. }\label{fig_2}
\end{figure}

\begin{table}[]
\begin{tabular}{|l|l|l|l|}
\hline
 Policy& Policy No & Policy Values & $\eta$ values  \\ \hline
$ [1 1 1 1] [1 1 1 1]$& No.256 & 0.180 &  $>1.0625\times10^{-4}$  \\ \hline
 $[0 1 1 1] [1 1 0 0]$& No.125 & 0.187 & $1.0000\times10^{-4}$   \\ \hline
 $[0 0 0 1][1 1 0 0]$& No.29 & 0.207 & $8.1250\times10^{-5}$ \\ \hline
 $[0 0 0 0][1 1 0 0]$& No.13 & 0.210 & $6.2500\times10^{-5}$ \\ \hline
 $\frac{[0 0 0 0][1 1 0 0]}{[1 0 1 0][1 1 1 0]}$& No.13 vs No.175 & 0.210 vs 0.196 & $<3.1250 \times 10^{-5}$ \\ \hline
\end{tabular}
\caption{The numerical analysis for the Revised Policy Improvement procedure.}
\label{tb_1}
\end{table}

%


%

\section{Conclusion}
\label{sec:Conclusion}
In this paper, to build a RL based framework for stroke rehabilitation, we proposed a cooperative  Markov Decision Processes model for improving the co-adaptive learning processes for both patient and smart rehabilitation devices. We studied this model in the framework of multi-agent Reinforcement Learning, in which the most critical problem is the non-stationary of the co-adaptation during mutual learning of the two agents. We proposed several auto-switching strategies to ensure the convergence of the policy improvement process. Furthermore, to reduce the frequency of the policy improvement of the patient, we proposed a Revised Policy Improvement procedure, which can balance between policy improvement and the overall value loss.

The numerical study of co-adaptation between patients and smart rehabilitation devices has been performed based on the proposed policy improvement strategies. Our final goal is to implement optimal adaptive learning from the perspective of both patient and the smart rehabilitation robot in real experiments and with clinical settings.



%






\small

\bibliography{VO2}

\end{document}